\RequirePackage{silence}
\documentclass[twoside]{article}

\usepackage[accepted]{aistats2021}
\usepackage[round]{natbib}

\usepackage{times}

\usepackage[utf8]{inputenc} 
\usepackage[T1]{fontenc}    
\usepackage{hyperref}       
\usepackage{booktabs}       
\usepackage{amsmath}
\usepackage{amssymb}
\usepackage{amsthm}
\usepackage{amsfonts}       
\usepackage{nicefrac}       
\usepackage{microtype}      
\usepackage{enumitem}
\usepackage{array}
\usepackage{algorithm}
\usepackage{algorithmic}
\usepackage{todonotes}
\usepackage{bm}
\usepackage{bbm}
\usepackage{subfiles}
\usepackage{thmtools,thm-restate}
\usepackage{wrapfig}

\usepackage{makecell}

\definecolor{pearThree}{HTML}{FF69B4}
\definecolor{pearDark}{HTML}{2980B9}
\definecolor{pearDarker}{HTML}{F330DB}
\hypersetup{
	colorlinks,
	citecolor=pearDark,
	linkcolor=pearThree,
	urlcolor=pearDarker}

\newtheorem{theorem}{Theorem}
\newtheorem{lemma}[theorem]{Lemma}

\newtheorem{definition}[theorem]{Definition}

\DeclareMathOperator*{\argmax}{arg\,max}
\DeclareMathOperator*{\argmin}{arg\,min}

\newcommand{\II}[1]{\mathbbm{1}_{\left\{#1\right\}}}

\newcommand{\R}{\mathbb{R}}

\newcommand{\NN}{{\mathbb N}}

\newcommand{\E}{\mathbb{E}}
\newcommand{\EE}[1]{\mathbb{E}\left[#1\right]}

\newcommand{\PP}[1]{\mathbb{P}\left[#1\right]}

\newcommand{\norm}[1]{\left\|#1\right\|}

\newcommand{\abs}[1]{\left|#1\right|}

\newcommand*{\eqdef}{\triangleq}

\newcommand{\cA}{\mathcal{A}}
\newcommand{\cB}{\mathcal{B}}
\newcommand{\cC}{\mathcal{C}}

\newcommand{\cF}{\mathcal{F}}

\newcommand{\cI}{\mathcal{I}}

\newcommand{\cL}{\mathcal{L}}

\newcommand{\cO}{\mathcal{O}}
\newcommand{\tcO}{\widetilde{\cO}}

\newcommand{\cT}{\mathcal{T}}

\newcommand{\bc}{{\bf c}}

\newcommand{\bW}{{\bf W}}
\newcommand{\by}{{\bf y}}

\newcommand{\bx}{{\bf x}}

\newcommand{\balpha}{{\boldsymbol \alpha}}

\newcommand{\blambda}{{\boldsymbol \lambda}}

\newcommand{\bnu}{{\boldsymbol \nu}}

\newcommand{\bomega}{{\boldsymbol \omega}}
\newcommand{\bell}{{\boldsymbol \ell}}

\newcommand{\nothere}[1]{}

\newcommand{\dist}[2]{\nu^{(#1)}_{#2}}
\newcommand{\loss}[2]{\ell^{(#1)}_{#2}}
\newcommand{\bloss}[1]{\boldsymbol{\ell}^{(#1)}}
\newcommand{\hloss}[2]{\hat{\ell}^{(#1)}_{#2}}
\newcommand{\lstar}[2]{\ell^{(#1)}_{#2|\star}}
\newcommand{\blstar}[2]{\boldsymbol{\ell}^{#1}_{#2|\star}}

\newcommand{\hblhstar}[2]{\hat{\boldsymbol{\ell}}^{#1}_{#2|\hat{\star}}}

\newcommand{\lstarC}[2]{\ell^{#1}_{#2|\star}}

\newcommand{\wloss}[1]{\ell_{#1}}

\newcommand{\wlstar}[1]{\ell_{#1|\star}}

\newcommand{\hhwlstar}[1]{\hat{\ell}_{#1|\hat{\star}}}
\newcommand{\hhlstar}[2]{\hat{\ell}^{(#1)}_{#2|\hat{\star}}}

\newcommand{\Loss}[2]{L^{(#1)}_{#2}}
\newcommand{\Lstar}[2]{L^{(#1)}_{#2|\star}}

\newcommand{\bLoss}[1]{\boldsymbol{L}_{#1}}
\newcommand{\bl}{{\boldsymbol{\ell}}}

\newcommand{\y}[3]{y^{(#1)}_{#2,#3}}
\newcommand{\wy}[2]{y_{#1,#2}}

\renewcommand{\b}[1]{\mathbb{#1}}

\newcommand{\hi}{\hat{\imath}}
\newcommand{\s}[1]{\sqrt{#1}}

\newcommand{\Tstar}{T^\star}
\newcommand{\istar}{i^\star}
\newcommand{\Alt}{\mathrm{Alt}}
\newcommand{\Alti}[2]{{\lambda}^{(#1)}_{#2}}

\newcommand{\LCB}{\mathrm{LCB}}
\usepackage{xspace}

\newcommand{\DT}{\texttt{D-Tracking}\xspace}

\newcommand{\CP}{\textcolor{red}{\texttt{CP}}\xspace}
\newcommand{\CG}{\textcolor{red}{\texttt{CG}}\xspace}

\newcommand{\AH}{\texttt{AdaHedge}\xspace}

\begin{document}

\twocolumn[

\aistatstitle{Stochastic bandits with vector losses: Minimizing $\ell^\infty$-norm of relative losses}

\aistatsauthor{ Xuedong Shang \And Han Shao \And Jian Qian }

\aistatsaddress{ Inria Lille, SequeL Team \And Toyota Technological Institute at Chicago \And MIT } ]

\begin{abstract}
Multi-armed bandits are widely applied in scenarios like recommender systems, for which the goal is to maximize the click rate. However, more factors should be considered, e.g., user stickiness, user growth rate, user experience assessment, etc. In this paper, we model this situation as a problem of $K$-armed bandit with multiple losses. We define relative loss vector of an arm where the $i$-th entry compares the arm and the optimal arm with respect to the $i$-th loss. We study two goals: (a) finding the arm with the minimum $\ell^\infty$-norm of relative losses with a given confidence level (which refers to fixed-confidence best-arm identification); (b) minimizing the $\ell^\infty$-norm of cumulative relative losses (which refers to regret minimization). For goal (a), we derive a problem-dependent sample complexity lower bound and discuss how to achieve matching algorithms. For goal (b), we provide a regret lower bound of $\Omega(T^{2/3})$ and provide a matching algorithm.

\end{abstract}

\section{Introduction}\label{sec:intro}


Multi-armed bandit is a classical sequential decision-making problem, where an agent/learner sequentially chooses actions (also called ``arms'') and observes a stochastic scalar loss of the chosen arm for $T$ rounds~\citep{thompson1933}. The two classical goals are to identify the best arm (which is the arm with the minimum expected loss) and to minimize the cumulative losses. Practical applications of multi-armed bandit, among many others, range from recommendation systems~\citep{zeng2016online}, clinical trials~\citep{durand2018contextual} to portfolio management~\citep{huo2017risk}. For example, when a user comes to an e-commerce website, traditional recommender systems choose a product (an arm) to recommend and observe whether the user clicks it or not (the loss). However, in addition to click rates, other factors like user stickiness should be considered as well in practice. Another example is fairness in public policy making, where each policy (an arm) can have drastic impacts over different gender/race groups (vector losses). These problems can be modeled as multi-armed bandits with \emph{vector/multi-dimensional} losses. In each dimension $i$, we measure the performance of an arm $k$ by comparing its $i$-th loss with the minimum loss among the $i$-th dimension, which we call relative loss\footnote{One may notice that the relative loss coincides with the traditional definition of regret in the scalar case.}. 

We provide a simple problem instance shown in Table~\ref{tab:instance1} to explain the intuition on how our setting differs from the usual one: the bandit model contains 3 arms, each row corresponds to the loss vector incurred by playing each arm and each column corresponds to the vector of absolute $i$-th losses for each arm. In this example, the optimal arm with respect to each arm has zero-loss, thus the relative losses coincide with the absolute losses. The optimal arm would be arm $3$ since it minimizes the maximum of each row. We formalize the intuition later in Section~\ref{sec:formulation}.

\begin{wrapfigure}[8]{r}{0.2\textwidth}
    \flushleft
    \begin{minipage}{0.2\textwidth}
        \vspace{-30pt}
        \begin{table}[H]
        \centering
        \caption{An instance.}
        \begin{tabular}{c|c|c}
        \toprule
        arms & $\bloss{1}$ & $\bloss{2}$\\
        \cmidrule{1-3}
        arm $1$ & $1$ & $0$\\
        \cmidrule{1-3}
        arm $2$ & $0$ & $1$\\
        \cmidrule{1-3}
        arm $3$ & $1/2$ & $1/2$\\
        \bottomrule
        \end{tabular}
        \label{tab:instance1}
        \end{table}
    \end{minipage}
    \vspace{-20pt}
\end{wrapfigure}

In this paper, we study both the two classical goals under the vector-loss setting: best-arm identification and regret minimization.

 Best-arm identification, as a particular type of \emph{pure exploration}, only cares about identifying the optimal arm given some stopping criterion. Two kinds of stopping criterion exist: (a) \emph{fixed-budget} for which the algorithm stops when a given budget is exhausted~\citep{bubeck2009pure,audibert2010budget,gabillon2012ugape,karnin2013sha,carpentier2016budget}; (b) \emph{fixed-confidence} for which the algorithm stops when we are able to spot the best arm with a high confidence level~\citep{even-dar2003confidence,kalyanakrishnan2012lucb,gabillon2012ugape,jamieson2014lilucb,garivier2016tracknstop,qin2017ttei,yu2018heavy,degenne2019game,menard2019lma,shang2020t3c}. In this paper, we focus on the second type and the detailed setting is described in Section~\ref{sec:formulation.bai}.

Contrary to best-arm identification, the objective of regret minimization, as indicated by its name, is to minimize the \emph{regret}: the gap between the total reward gathered by the agent and the cumulative reward obtained by optimal strategy. Regret minimization naturally balances between exploration and exploitation. An asymptotic lower bound on the regret is given by~\cite{lai1985}. Since then the problem has been extensively studied. Typical solutions include optimistic algorithms~\citep{auer2002ucb,cappe2013klucb,honda2015imed}, their Bayesian competitor Thompson sampling~\citep{thompson1933,kaufmann2012thompson,agrawal2013further,korda2013thompson}, and non-parametric methods~\citep{baransi2014besa,chan2020ssmc,baudry2020}. In our paper, the objective is somehow different. We aim to minimize the $\ell^\infty$-norm of cumulative relative loss, which requires a more specific definition of regret that we give in Section~\ref{sec:formulation.regret}.


\paragraph{Related work.}

Vector payoffs/losses, as a core ingredient of this work, mostly finds its popularity among literature of online learning, in particular in a game theory point of view. The problem is closely related to multi-objective optimization where the use of \emph{Blackwell approachability} has been thoroughly investigated both for the \emph{full information} setting~\citep{perchet2014blackwell} and the \emph{partial monitoring} setting~\citep{kwon2017blackwell,perchet2011blackwell}. The very same problem is less studied for multi-armed bandit. To the best of our knowledge, minimizing the $\ell^\infty$-norm of (cumulative) relative loss has never been looked into in the bandit literature. For best-arm identification, a related setting refers to to~\cite{kats-samuels2019top}, where the feedback is also multi-dimensional, but the goal is constrained maximization. For regret minimization, the most similar setting to ours is the multi-objective multi-armed bandit that considers conflicting sub-objectives. It is first proposed by~\cite{drugan2013} and~\cite{zuluaga2013}, and is followed by a series of extensions~\citep{auer2016pareto,drugan2014,lu2019glb}. Multi-objective multi-armed bandit aims to find the \emph{Pareto frontier} of different sub-objectives, while our setting only cares about the maximum. For example, arm $1$ and arm $2$ in Table~\ref{tab:instance1} are on the Pareto frontier as well, but do not achieve optimality in our definition.

\paragraph{Contributions.}

The contributions of this paper are the following:
\begin{itemize}
    \item We describe a novel multi-armed bandit setting with $d$-dimensional vector losses and we study the problem in both best-arm identification and regret minimization. We design the performance measure as minimizing the $\ell^\infty$-norm of relative loss over all single dimensions. 
    \item We first investigate best-arm identification. We derive a \emph{problem-dependent} lower bound on the sample complexity and discuss how to achieve matching algorithms for \emph{fixed-confidence} best-arm identification.
    \item We then study regret minimization. We show that any algorithm suffers a \emph{worst-case} regret of order $\Omega(T^{2/3})$ under our setting. We provide an algorithm based on \emph{two-player game} with matching upper regret bound up to a log factor.
    
\end{itemize}


\paragraph{Outline.}

The rest of the paper is organized as follows. We start by the problem formulation in Section~\ref{sec:formulation} where we specify both best-arm identification and regret minimization under our setting. We first study best-arm identification in Section~\ref{sec:bai} for which we focus on the sample complexity. It then follows regret minimization in Section~\ref{sec:lb} where we provide the worst-case lower bound along with a simple matching algorithm before we conclude.


\section{Problem formulation}\label{sec:formulation}

Our model $\bnu$ for the environment is a $K$-armed bandit with \emph{unknown} vector payoffs, i.e., vector loss distributions $(\dist{1}{k},\ldots,\dist{d}{k})_{k\in [K]}$ where $\dist{i}{k}$ is the $i$-th sub-(scalar) loss distribution for the $k$-th arm. Each distribution $\dist{i}{k}$ is from a known sub-Gaussian canonical exponential family with one parameter (the mean of the distribution) for all $i$ and $k\in [K]$.

We consider two mainstream multi-armed bandit frameworks in this paper (see~\citealt{kaufmann2017survey} for a survey), namely best-arm identification and regret minimization. In both settings, a learning algorithm $\cA$ selects an arm $\cA_t\in [K]$ at each round $t=1,\ldots,T$, and then observes a loss vector of arm $k$: $\wy{\cA_t}{t}\sim (\dist{1}{\cA_t},\ldots,\dist{d}{\cA_t})$. Let $\cF_{t}=\sigma (\cA_1,y_{\cA_1,1},\ldots, \cA_t,\wy{\cA_t,t})$ be the information available to the algorithm after $t$ rounds. We specify the two frameworks under our setting in this section.

\subsection{Some notations}

Let $\Sigma_n \eqdef \{\bomega\in [0,1]^n : \sum_{i=1}^n \omega_i = 1\}$ with $n\in \NN$ denote the $n$-dimensional probability simplex. Let $\mathbf{1}$ denote the all-one vector whose dimension can be known from the context. We let $d(x,y)$ denote the Kullback-Leibler divergence from the distribution parameterized by $x$ to that parameterized by $y$ for $x,y\in [0,1]$. We let $d^+(x,y) = d(x,y)\II{x>y}$. For simplicity, we abuse $\argmin$ (resp. $\argmax$) to represent an arbitrary element that achieves the minimum (resp. maximum) and fix this element thereafter\footnote{It is not hard to check that the choice does not affect the results in this paper.}. We introduce several notions of loss for problem formulation. Note that all the following loss definitions depend on the bandit model $\bnu$. For simplicity, we omit it in the notations whenever there is no ambiguity.

For each $i\in[d]$, we define the $i$-th \emph{expected loss} as
\[
    \bloss{i} \eqdef (\loss{i}{1},\ldots,\loss{i}{K})\,,
\]
where $\loss{i}{k} = \mathbb{E}[\dist{i}{k}] \in [0,1]$ for $k\in[K]$. Similarly, we denote by 
\[
    \bell_k \eqdef (\loss{1}{k},\ldots,\loss{d}{k})
\]
the \emph{expected loss} vector of arm $k$. A bandit model in this paper can thus be interchangeably represented by $\bnu$ or $\bl = (\bell_1,\ldots,\bell_K)\in [0,1]^{d\times K}$.

Let $ \star_i\eqdef\argmin_{k\in [K]} \loss{i}{k}$  denote the index of the arm  with the lowest $i$-th expected loss. We further define the $i$-th \emph{expected relative loss} for $i\in[d]$ as
\[
    \blstar{(i)}{}=(\lstar{i}{1},\ldots,\lstar{i}{K})\,,
\]
and the \emph{expected relative loss} of arm $k$ as 
\[
    \blstar{}{k} \eqdef ( \lstar{1}{k},\ldots,\lstar{d}{k} )\,,
\]
where $\lstar{i}{k}\eqdef\loss{i}{k}- \loss{i}{\star_i}$. And we denote the matrix of the expected relative losses by
\[
    \blstar{}{}=(\blstar{}{1},\ldots,\blstar{}{K})\in [0,1]^{d\times K}\,.
\]
We define the $i$-th \emph{expected loss of weight $\bomega\in \Sigma_K$} as
\[
    \loss{i}{\bomega} \eqdef\bomega^\top \bloss{i}\,,
\]
and the $i$-th \emph{expected relative loss of the weight $\bomega$} as 
\[
    \lstar{i}{\bomega} \eqdef\bomega^\top \blstar{(i)}{}\,.
\]
Finally, we denote by 
\[
    \wlstar{\bomega} \eqdef \norm{(\lstar{1}{\bomega},\ldots,\lstar{d}{\bomega})}_\infty
\]
the \emph{$\ell^\infty$-norm of the expected relative loss of the weight $\bomega$}.


\subsection{Best-arm identification}\label{sec:formulation.bai}

We first detail the framework of fixed-confidence best-arm identification in our case: the objective is to identify the arm with the minimum relative loss in terms of infinite norm. That is, for each bandit model $\bl\in [0,1]^{d\times K}$, the \emph{unique} correct answer is given by 
\[
\istar(\bl) \eqdef  \argmin_{k\in[K]} \|\blstar{}{k}\|_\infty   = \argmin_{k\in[K]} \max_{i\in[d]} \lstar{i}{k}
\]
among the set of possible correct answers $\cI = [K]$.

\paragraph{Motivation.}
In general, the vector-loss/payoff settings considered by previous works mainly focus on the Pareto frontier of different sub-objectives. This notion of optimality is unreasonable in some cases, where some dimensional losses suffer extremely high scalar regrets. To avoid the risk of incredibly high scalar regrets for any single dimension, we target at minimizing the infinite norm of the relative losses (which are scalar regrets) and thus, we can bound the scalar regrets for all dimensions at the same time.

\paragraph{Algorithm.}
A deterministic pure-exploration algorithm under the fixed-confidence setting is given by three components: (1) a \emph{sampling rule} $(\cA_t)_{t\geq 1}$, where $\cA_t\in[K]$ is $\cF_{t-1}$-measurable. (2) a \emph{stopping rule} $\tau_\delta$, a stopping time for the filtration $(\cF_t)_{t\geq 1}$, and (3) a \emph{decision rule} $\hi\in \cI$ which is $\cF_{\tau_\delta}$-measurable.
Non-deterministic algorithms could also be considered by allowing the rules to depend on additional internal randomization. The algorithms we present are deterministic.

\paragraph{$\delta$-correctness and fixed-confidence objective.}
An algorithm is $\delta$-correct if it predicts the correct answer with probability at least $1-\delta$, precisely if $\mathbb{P}_\bl \big(\hi \neq \istar(\bl)\big) \leq \delta$ and $\tau_\delta < +\infty$ almost surely for all $\bl \in [0,1]^{d\times K}$. The goal is to find a $\delta$-correct algorithm that minimizes the \emph{sample complexity}, that is, the expected number of samples $\E_\bl[\tau_\delta]$ needed to predict an answer.

\subsection{Regret minimization}\label{sec:formulation.regret}

We now detail the setting for regret minimization. Let $\bLoss{\cA} \eqdef \sum_{t=1}^T \boldsymbol{\ell}_{\cA_t}$ denote the expected cumulative loss of algorithm $\cA$ where $\bell_{\cA_t}=(\loss{1}{\cA_t},\ldots,\loss{d}{\cA_t})$ and $L_{\cA}^{(i)}\eqdef\sum_{t=1}^T \loss{i}{\cA_t}$ be the expected cumulative losses. The traditional regret (which we call relative loss) w.r.t. the scalar loss $\loss{i}{\star_i}$ is defined as
\begin{align*}
    \Lstar{i}{\cA} \triangleq \Loss{i}{\cA}-T\loss{i}{\star_i}
\end{align*}
for $i\in [d]$. The goal is to minimize the $\ell^\infty$-norm of cumulative relative loss, which differs from the goal of classical stochastic multi-armed bandits with scalar payoffs. However, comparing the $\ell^\infty$-norm of cumulative relative loss of an algorithm with a single optimal arm may be unreasonable. For example, a bandit problem with three arms $(1,0)$, $(0,1)$ and $(3/4,3/4)$ has the optimal arm $(3/4,3/4)$. But we can achieve $\ell^\infty$-norm of cumulative relative loss $T/2$ by pulling arm $1$ and arm $2$ for $T/2$ rounds respectively while always pulling the single optimal arm can only achieve $3T/4$. Therefore, it is more reasonable to look into the optimal proportion of arm pulls instead of only considering the single optimal arm under the context of vector losses. We call the optimal proportion of arm pulls the \emph{optimal weight}.



\begin{definition}[optimal weight]
We define
\begin{align*}
    \bomega^\star &\triangleq \argmin_{\bomega\in\Sigma_K}\wlstar{\bomega}
\end{align*}
the optimal weight of arms.
\end{definition}

Consequently, it is also natural to measure the performance by comparing with the optimal weight. Therefore, we introduce the following regret in terms of the relative losses defined w.r.t. the optimal weight.

\begin{definition}[regret]
The expected regret of algorithm $\cA$ is defined as
\begin{align}
\b{E} [R_\cA(T)] \triangleq \b{E} \left[\norm{\bLoss{\cA|\star}}_\infty\right]-\wlstar{\bomega^\star}T\,,
\end{align}
where $\bLoss{\cA|\star}=(\Lstar{1}{\cA},\ldots,\Lstar{d}{\cA})$.
\end{definition}

\section{Best-arm identification}\label{sec:bai}

We first study best-arm identification for our setting in a fixed-confidence context. We are thus interested in the sample complexity. We begin with particularizing the general problem-dependent lower bound by~\cite{garivier2016tracknstop} to our setting. Then we discuss how to design asymptotically optimal algorithms that we precise the definition.

\subsection{Lower bound on the sample complexity}

We first derive a problem-dependent lower bound as stated in the following theorem.

\begin{theorem}\label{thm:lb_bai}
Let $S_y(\eta)\eqdef \{i|\eta_i\leq y_i\}$ and $C_\gamma(z)\eqdef  \{i|z_i\leq \gamma\}$ for $\eta,y\in \R^d$, $z\in \R^K$ and $\gamma\in \R$. For any $\delta$-correct strategy and any $\bl\in [0,1]^{d\times K}$, we have
\begin{equation*}
  \label{eq:lb_general}
  \liminf_{\delta\to 0}\frac{\E_\bl[\tau_{\delta}]}{\log(1/\delta)} \geq \Tstar(\bl) \,,
\end{equation*}
where $\Tstar(\bl)$ is a characteristic time defined by
\begin{align}
    &\Tstar(\bl)^{-1}\nonumber\\
    =& \max_{\bomega \in \Sigma_K} \min_{\substack{k^\star\in [K]\\j\in[d]}}\inf_{\substack{x\in [0,1],y\in [0,1]^d:\\ y \leq (1-x)\mathbf{1}}} \omega_{\istar(\bl)}d(\loss{j}{\istar(\bl)}, x+y_j) \nonumber\\
    &+ \omega_{k^\star}\left(\sum_{i\notin S_y(\bl_{k^\star})}d^+(\loss{i}{k^\star}, x+y_i)\right)\nonumber\\ 
    &+ \sum_{i\neq j}\sum_{k\in C_{y_i}(\bloss{i})}\omega_i d(\loss{i}{k},y_i)\nonumber\\
    &+ \sum_{k\in C_{y_j}(\bloss{j})/\{\istar(\bl)\}}\omega_j d(\loss{j}{k},y_j)
    \,.\label{eq:chartime}
\end{align}

\end{theorem}
\begin{proof}
By Theorem~1 of~\cite{garivier2016tracknstop}, we have 
\begin{equation*}
  \liminf_{\delta\to 0}\frac{\E_\bl[\tau_{\delta}]}{\log(1/\delta)} \geq \Tstar(\bl) \,,
\end{equation*}
where 
\[\Tstar(\bl)^{-1} = \max_{\bomega \in \Sigma_K}\inf_{\bm\lambda \in\Alt(\bl)} \left( \sum_{k\in [K]}\sum_{i\in [d]} \omega_k d(\loss{i}{k},\Alti{i}{k})\right)\;,
\]
where $\Alt(\bl)$ is an alternative bandit problem with different optimal arm, i.e.,
\[\Alt(\bl) \eqdef \{ \bm\lambda \in [0,1]^{d\times K}: \istar(\bm\lambda) \neq \istar(\bl)\}\;.\]
Then we just need to calculate $\Tstar(\bl)^{-1}$ to complete the proof.
For an alternative $\bm\lambda$ with $\istar(\bm\lambda) = k^\star$, we let $y=(\Alti{1}{\star_1},\ldots,\Alti{d}{\star_d})\in [0,1]^d$ and $x=\max_i{\Alti{i}{k^\star} - y_i}\in [0,1]$. Then we have $\exists j\in[d], \Alti{j}{\istar(\bl)}-y_j\geq x$. For any $\bomega\in\Sigma_K$, we have
\begingroup
\allowdisplaybreaks
\begin{align*}
    &\inf_{\bm\lambda \in\Alt(\bl)} \left( \sum_{k\in [K]}\sum_{i\in [d]} \omega_k d(\loss{i}{k},\Alti{i}{k})\right)\\
    =&\min_{k^\star\in[K]} \inf_{\substack{\blambda \in\Alt(\bl):\\ \istar(\bm\lambda) = k^\star}} \left( \sum_{k\in [K]}\sum_{i\in [d]} \omega_k d(\loss{i}{k},\Alti{i}{k})\right)\\
    =& \min_{k^\star\in[K]} \inf_{\substack{\blambda \in\Alt(\bl):\\ \istar(\bm\lambda) = k^\star}} \left(\sum_{\substack{k\in \cup_{i\in[d]} C_{y_i}(\bloss{i})\\ \cup \{\istar(\bl),k^\star\}}} \sum_{i\in [d]} \omega_k d(\loss{i}{k},\Alti{i}{k})\right)\\
    =& \min_{\substack{k^\star\in [K]\\j\in[d]}}\inf_{\substack{x\in [0,1],y\in [0,1]^d:\\ y \leq (1-x)\mathbf{1}}} \omega_{\istar(\bl)}d(\loss{j}{\istar(\bl)}, x+y_j) \\
    &+ \omega_{k^\star}\left(\sum_{i\notin S_y(\wloss{k^\star})}d^+(\loss{i}{k^\star}, x+y_i)\right)\\ 
    &+ \sum_{i\neq j}\sum_{k\in C_{y_i}(\bloss{i})}\omega_i d(\loss{i}{k},y_i)\\
    &+ \sum_{k\in C_{y_j}(\bloss{j})/\{\istar(\bl)\}}\omega_j d(\loss{j}{k},y_j)
    \,,
\end{align*}
\endgroup
which completes the proof.
\end{proof}

\subsection{Asymptotically optimal algorithms}

To design an algorithm for fixed-confidence best-arm identification, one needs to specify three components as previously mentioned: a stopping rule, a decision rule and a sampling rule. The \emph{Track-and-Stop} strategy proposed by~\cite{garivier2016tracknstop} can be adopted in our setting with optimal sample complexity. For completeness, we describe the algorithm briefly below.

In the next, we use the empirical average $\hat{\bl}_t$ to estimate the expected losses $\bl$ at time $t$, that is 
\[
    \forall k\in[K], i\in[d], \hloss{i}{t,k} \eqdef \frac{1}{t}\sum_{\tau =1}^{t}\y{i}{\tau}{k}\,.
\]

\paragraph{Decision rule.}

Let $f(\cdot)$ be a function of time-dependent exploration bonus (e.g. $\log(t)$) for $t\in\NN$. Let $[c_{t,k},d_{t,k}]\eqdef\{\xi:\sum_{i\in [d]} N_{t-1,k} d(\hloss{i}{t-1,k},\xi)\leq f(t)\}$. Now, let
{\small
\[
    \tilde{\bl}_{t-1}\eqdef\argmin_{\substack{\blambda \in [0,1]^{d\times K}\cap \\ \prod_{k=1}^K [c_{t,k},d_{t,k}]^d}} \left( \sum_{k\in [K]}\sum_{j\in [d]} N_{t-1,k} d(\hloss{j}{t-1,k},\Alti{j}{k})\right)\,.
\]
}
Then for the decision rule, we choose to recommend $\hi = \istar(\tilde{\bl}_{t-1})$.  
Note that if the empirical loss matrix $\hat{\bl}_{t-1}\in[0,1]^{d\times K}$, then $\tilde{\bl}_{t-1}$ coincides with $\hat{\bl}_{t-1}$ and the decision is simply the empirical best arm.

\paragraph{Stopping rule.}

In this paper, we choose to use the classical Chernoff stopping rule (see e.g.~\citealt{chernoff1959,garivier2016tracknstop}) that can be concretized (for exponential family bandit models) to the following form:
\begin{equation*}\label{eq:chernoffstoppingtime}
\tau_\delta\eqdef \inf \left\lbrace t \in \mathbb{N} : \operatorname{GLR}_t(\Alt(\hat{\bl}_t)) > \beta(t,\delta) \right\rbrace\,,
\end{equation*}
where $\beta(t,\delta)$ is a threshold function to be chosen carefully and
\[
    \operatorname{GLR}_t(\Alt(\hat{\bl}_t)) = \inf_{\blambda \in\Alt(\hat{\bl}_t)} \left( \sum_{k\in [K]}\sum_{i\in [d]} N_{t,k} d(\hloss{i}{t,k},\Alti{i}{k})\right)
\]
is the \emph{generalized log-likelihood ratio} between the alternative set $\Alt(\hat{\bl}_t)$ and the whole parameter space.

Using the same reasoning as~\cite{shang2020t3c}, one can show that the Chernoff stopping rule coupled with the threshold
\begin{equation*}
\beta(t,\delta) \eqdef 4\log(4+\log(t)) + 2 \cC\left(\frac{\log((Kd-1)/\delta)}{2}\right)
\end{equation*} leads to the $\delta$-correctness, i.e. $\PP{\tau_{\delta} < \infty \wedge \hi \neq \istar(\bl)} \leq \delta$ for any sampling rule. The function $\cC$ is given by~\cite{kaufmann2018mixture} that satisfies $\cC(x) \simeq x+\log(x)$. Note that in practice, one can simply choose to set $\beta(t,\delta)=\log((1+\log(t))/\delta)$.

\paragraph{Sampling rule and the whole picture.}

We aim to design algorithms that match the lower bound derived in Theorem~\ref{thm:lb_bai}. We call such algorithms asymptotically optimal. Formally, a fixed-confidence algorithm is asymptotically optimal if
\[
\limsup_{\delta \rightarrow 0}\frac{\E_\bl[\tau_{\delta}]}{\log(1/\delta)} \leq \Tstar(\bl)\,.
\]

To achieve this property, the learner needs to allocate her pulls according to the optimal weight vector given by the characteristic time~\citep{garivier2016tracknstop,russo2016ttts}, that is
\begingroup
\allowdisplaybreaks
\begin{align}
    &\bomega^\star(\bl)=\argmax_{\bomega \in \Sigma_K}\inf_{\bm\lambda \in\Alt(\bl)} \left( \sum_{k\in [K]}\sum_{i\in [d]} \omega_k d(\loss{i}{k},\Alti{i}{k})\right)\nonumber\\
    =& \argmax_{\bomega \in \Sigma_K} \min_{\substack{k^\star\in [K]\\j\in[d]}}\inf_{\substack{x\in [0,1],y\in [0,1]^d:\\ y \leq (1-x)\mathbf{1}}} \omega_{\istar(\bl)}d(\loss{j}{\istar(\bl)}, x+y_j) \nonumber\\
    &+ \omega_{k^\star}\left(\sum_{i\notin S_y(\bl_{k^\star})}d^+(\loss{i}{k^\star}, x+y_i)\right)\nonumber\\ 
    &+ \sum_{i\neq j}\sum_{k\in C_{y_i}(\bloss{i})}\omega_i d(\loss{i}{k},y_i)\nonumber\\
    &+ \sum_{k\in C_{y_j}(\bloss{j})/\{\istar(\bl)\}}\omega_j d(\loss{j}{k},y_j)\,,\label{eq:computation}
\end{align}
\endgroup
which can be considered as solving a minimax saddle-point problem. Although the $\inf$ part is non-convex, it is computable by calculating the infimum over $x,y$ for each $k^\star$ and $j$. To calculate the infimum over $x,y$, for each $i\in [d]$, we consider the case that $y_i$ is larger than the $i$-th losses of $m_i$ arms with $m_i = 0,1,\ldots,K$ separately. In each case of fixed $k^*,j,\{m_i\}_{i\in[d]}$, the infimum part of~\eqref{eq:computation} is convex and solvable. However, this incurs a computational complexity of $\Theta(dK^{d+1})$.

The aforementioned problem requires the knowledge of the true means, one simple way to overcome this is to adopt the \DT rule~\citep{garivier2016tracknstop}, where we choose to sample
\[
    \cA_{t+1} \in \argmax_{k\in[K]} \omega^\star_{k}(\hat{\bl}_t)-N_{t,k}/t\,
\]
using `plug-in' estimates of the optimal weight. \DT is proved to be asymptotically optimal~\citep{garivier2016tracknstop}, with a known drawback as its computational liability due to the optimization problem~\eqref{eq:computation} that has to be treated once at each step, since there is no known closed form expression or even no computationally feasible approximation approach in general. 

An improved algorithm without solving the optimization problem every round by solving a two-player game derived from~\cite{degenne2019game} is given in Appendix~\ref{app:bai}.

\section{Regret minimization}\label{sec:lb}

We turn our attention to regret minimization. We first derive a worst-case lower bound. Then we present an efficient algorithm that matches the lower bound. For simplicity, we omit the $t$ in the subscripts in this section, e.g., we denote $\hloss{i}{k}$ instead of $\hloss{i}{t,k}$.

\subsection{Worst-case lower bound}

\begin{restatable}{theorem}{restatelb}\label{thm:lb}
    For $T>27$, let $\sup$ be the supremum over all distributions of losses and $\inf$ be the infimum over all algorithms. Then we have,
    \begin{align*}
    \inf_{\cA}\sup_{\bnu}\mathbb{E}\left[ R_{\cA}(T)\right] \geq \frac{1}{2304} T^{\frac{2}{3}}\,.
    \end{align*}
\end{restatable}


\begin{proof}
    Let $\epsilon \in [0,1/6]$ be a constant, we consider a bandit model $\bnu$ with the following 2-dimensional loss vectors:
    \begin{gather*}
    \bell_1 = \left(\frac{1}{4},\frac{3}{4}\right), \bell_2 = \left(\frac{3}{4},\frac{1}{4}\right),\\
    \bell_3 = \left( \frac{3-\epsilon}{8}, \frac{3+\epsilon}{8}\right), \bell_4 = \left(  \frac{3+\epsilon}{8},  \frac{3-\epsilon}{8}\right)\,,
    \end{gather*}
    where the losses are Gaussian distributions with variance 1 and expectation of the indicated value. 

    Denote $N_1,N_2,N_3$ and $N_4$ the number each arm is pulled. Since there is a symmetry between arm 3 and arm 4. Without loss of generality, we assume for the given algorithm $\cA$ under consideration, that $\mathbb{E}_{\cA,\bnu}\left[N_3\right]\leq \mathbb{E}_{\cA,\bnu}\left[N_4\right]$.
    Then according to the assumption between $N_3$ and $N_4$, we consider an alternative bandit model $\bnu'$ with the following losses,
    \begin{gather*}
        \bell'_1 = \left(\frac{1-\epsilon}{4},\frac{3}{4}\right), \bell'_2 = \left(\frac{3}{4},\frac{1}{4}\right),\\
        \bell'_3 = \left( \frac{3-\epsilon}{8}, \frac{3+\epsilon}{8}\right), \bell'_4 = \left(  \frac{3+\epsilon}{8},  \frac{3-\epsilon}{8}\right)\,.
    \end{gather*}
    ~\\
    $\text{For } \epsilon < 1/6:$  The optimal arms for each loss are $\star_1 = 1$, $\star_2 = 2$.
    \begin{align*}
        \omega^\star_{\bnu'} 
        &\eqdef \argmin_{\bomega\in\Sigma_K} \max_{i\in[d]} \left\{\bomega^\top (\bell')^{(i)} -\loss{i}{\star_i,\bnu'}\right\} \\
        &= \argmin_{\bomega\in\Sigma_K} \max \biggl\{ \frac{(2+\epsilon)\omega_2}{4} + \frac{(1+\epsilon)\omega_3}{8}\\
        &\quad+ \frac{(1+3\epsilon)\omega_4}{8}, \frac{\omega_1}{2} + \frac{(1+\epsilon)\omega_3}{8} + \frac{(1-\epsilon)\omega_4}{8} \biggr\}\\
        &=(0,0,1,0).
    \end{align*}
    Thus the regret is lower bounded as follows,
    \begingroup
    \allowdisplaybreaks
    \begin{align*}
        R_{\bnu'}(T) 
        &\triangleq \max_{i\in [d]}\left(\Loss{i}{}- \loss{i}{\bomega_{\bnu'}^\star}T\right) \\
        &= \max \biggl\{  \frac{(1-\epsilon)N_1}{4} + \frac{3N_2}{4} + \frac{(3-\epsilon)N_3}{8} \\
        &\quad + \frac{(3+\epsilon)N_4}{8} - \frac{3-\epsilon}{8}T, \frac{3N_1}{4} + \frac{N_2}{4} \\
        &\quad + \frac{(3+\epsilon)N_3}{8} + \frac{(3-\epsilon)N_4}{8} - \frac{3+\epsilon}{8}T \biggr\}\\
        &= \max  \biggl\{  -\frac{(1+\epsilon)N_1}{8} + \frac{(3+\epsilon)N_2}{8}  + \frac{\epsilon N_4}{4},\\
        &\qquad\qquad \frac{(3-\epsilon)N_1}{8} - \frac{(1+\epsilon)N_2}{8}  - \frac{\epsilon N_4}{4} \biggr\}\\
        &\geq \frac{2}{3} \left(   -\frac{(1+\epsilon)N_1}{8} + \frac{(3+\epsilon)N_2}{8}  + \frac{\epsilon N_4}{4} \right) \\
        &+ \frac{1}{3} \left(  \frac{(3-\epsilon)N_1}{8} - \frac{(1+\epsilon)N_2}{8}  - \frac{\epsilon N_4}{4}  \right)\\
        &\geq \frac{1}{48}N_1 + \frac{5}{24}N_2 + \frac{\epsilon}{12} N_4 \\
        &= (\frac{1}{48} - \frac{\epsilon}{12}  )N_1 + (\frac{5}{24} - \frac{\epsilon}{12} )N_2 +  \frac{\epsilon}{12}(T - N_3) \\
        &\geq \frac{1}{144}N_1 + \frac{1}{6}N_2 +   \frac{\epsilon}{12}(T - N_3)\,.
    \end{align*}
    \endgroup
    ~\\
    So we have the following regret for the bandit model $\bnu'$,
    \begin{align}
        \mathbb{E}\left[R_{\cA,\bnu'}(T)\right] 
        &\geq  \frac{1}{144}\mathbb{E}_{\cA,\bnu'}\left[N_1\right] + \frac{\epsilon}{12} \left(T - \mathbb{E}_{\cA,\bnu'}\left[N_3\right]\right)\,.\label{eq:lbalter}
    \end{align}
    According to the inequality (6) by~\cite{garivier2018explore}, we have,
    \begin{align*}
    \frac{\epsilon^2}{32} \mathbb{E}_{\cA,\bnu'}\left[N_1\right] &\geq \text{kl} \left(\frac{ \mathbb{E}_{\cA,\bnu'}\left[N_3\right]}{T},\frac{ \mathbb{E}_{\cA,\bnu}\left[N_3\right]}{T}  \right) \\
    &\geq \frac{1}{2} \left( \frac{ \mathbb{E}_{\cA,\bnu}\left[N_3\right]}{T} - \frac{ \mathbb{E}_{\cA,\bnu'}\left[N_3\right]}{T} \right)^2\,.
    \end{align*}
    Therefore,
    \begin{align*}
        \mathbb{E}_{\cA,\bnu'}\left[N_3\right] \leq \frac{\epsilon}{4} T \s{\mathbb{E}_{\cA,\bnu'}\left[N_1\right]} + \mathbb{E}_{\cA,\bnu}\left[N_3\right]
    \end{align*}
    Furthermore with $\mathbb{E}_{\cA,\bnu}\left[N_3\right] \leq T/2$, and according to our assumption,
    \begin{align*}
    \mathbb{E}\left[R_{\cA,\bnu'}(T)\right]
    &\geq\frac{1}{144}\mathbb{E}_{\cA,\bnu'}\left[N_1\right] + \frac{\epsilon}{12} \left(T - \mathbb{E}_{\cA,\bnu'}\left[N_3\right]\right)\\
    &\geq \frac{1}{144}\mathbb{E}_{\cA,\bnu'}\left[N_1\right] \!+\! \frac{\epsilon}{12} T \!-\! \frac{\epsilon^2}{48}T\s{\mathbb{E}_{\cA,\bnu'}\left[N_1\right]} \\
    &\quad - \frac{\epsilon}{12} \mathbb{E}_{\cA,\bnu}\left[N_3\right] \\
    &\geq \frac{1}{144}\mathbb{E}_{\cA,\bnu'}\left[N_1\right]  \!+\! \frac{\epsilon}{24} T \!-\! \frac{\epsilon^2}{48}T\s{\mathbb{E}_{\cA,\bnu'}\left[N_1\right]} 
    \end{align*}
    Take $\epsilon = T^{-1/3}/2 < 1/6$, we have,
    \begin{align*}
    \mathbb{E}\left[R_{\cA,\bnu'}(T)\right] \geq \frac{1}{144}\mathbb{E}_{\bnu'}\left[N_1\right]  \!+\! \frac{1}{48} T^{\frac{2}{3}} \!-\! \frac{1}{192} T^{\frac{1}{3}} \s{ \mathbb{E}_{\bnu'}\left[N_1\right]} 
    \end{align*}
    If $\mathbb{E}_{\bnu'}\left[N_1\right]\geq T^{2/3}/16$, then by~\eqref{eq:lbalter}, we have,
    \begin{align*}
    \mathbb{E}\left[R_{\cA,\bnu'}(T)\right] \geq \frac{1}{144}\mathbb{E}_{\bnu'}\left[N_1\right] \geq \frac{1}{2304} T^{\frac{2}{3}} 
    \end{align*}
    Else, we have,
    \begin{align*}
    \mathbb{E}\left[R_{\cA,\bnu'}(T)\right] &\geq \frac{1}{144}\mathbb{E}_{\bnu'}\left[N_1\right]  \!+\! \frac{1}{48} T^{\frac{2}{3}} \!-\! \frac{1}{192} T^{\frac{1}{3}} \s{ \mathbb{E}_{\bnu'}\left[N_1\right]}  \\
        &\geq \frac{1}{48} T^{\frac{2}{3}} - \frac{1}{192} T^{\frac{1}{3}} \frac{1}{4} T^{\frac{1}{3}} \\
        &\geq \frac{1}{2304} T^{\frac{2}{3}} 
    \end{align*}

\end{proof}


\paragraph{A simple method derived from the lower bound proof.} The lower bound proof actually indicates that the minimum losses for each dimension are crucial in order to achieve optimality. To this regard, following a simple scheme of forced exploration, then exploit, we could easily derive an algorithm matching the lower bound for the minimax regret. Detailed description of the algorithm and analysis can be found in Appendix~\ref{app:cp}. Despite its simplicity, the computation complexity scales exponentially with $d$. To cope with this issue, we develop a second algorithm with the two-player game scheme.

\subsection{A minimax game}\label{sec:algos}

We propose an algorithm called \underline{C}ombinatorial \underline{G}ame (\CG), whose pseudo-code is displayed in Algorithm~\ref{alg:betterrm}.

The idea is to introduce a two-player game scheme as recently studied by~\cite{degenne2020structure}, where one tries to identify the best allocation of probability across the arms while the opponent always replies with a best response. More specifically, at each round $t$ we request from the first learner its probability allocation, and pull arms accordingly. When the losses are revealed, we calculate the fictitious losses the learner would have suffered if it had played the arm, and feed the fictitious losses to the learner, as displayed in Algorithm \ref{alg:betterrm}. The learner is supposed to have regret bounds similar to \AH~\citep{derooij2014hedge}.







Concretely, the arm with the smallest empirical $i$-th loss is denoted by $\hat{\star}_i\!\eqdef\!\argmin_{k\in[K]} \hloss{i}{k}$ for $i\in [d]$. Let $\operatorname{LCB}(\blstar{}{})$ be the lower confidence bound of $\blstar{}{}$, calculated as,
\begin{align}
    \LCB(\blstar{}{})_{i,k} =   \hhlstar{i}{k} - \sqrt{\frac{2\log(T)}{N_{t,k}}} - \sqrt{\frac{2\log(T)}{N}}\;, \label{eq:lcb}
\end{align}
where $\hloss{i}{k|\hat{\star}} = \hloss{i}{k} - \hloss{i}{\hat{\star}_i}$. Then we can define the best response in an optimistic fashion: $\bx_t=\argmax_{\bx\in \Sigma_d} \bx^\top \operatorname{LCB}(\blstar{}{}) \bomega_t$, and feed the optimistic loss $\operatorname{LCB}(\blstar{}{})^\top \bx_t$ back to $\cL$.

\begin{algorithm}[ht] \caption{The algorithm of \CG}\label{alg:betterrm}
{\begin{algorithmic}[1]
\STATE {\bfseries Input:} Time horizon $T$, number of forced exploration rounds $N$, learner $\cL$ for linear losses on the simplex
\STATE Pull each arm for $N$ rounds
\STATE Start an instance of $\cL$ and set $N_{1,k} = 0$ for all $k\in [K]$ 
\FOR{$t=1,\cdots,T-KN$}
	\STATE $\hat{\star}_i = \argmin_{k\in[K]} \hloss{i}{k}$ for $i\in [d]$
	\STATE Get $\bomega_t$ from $\cL$
	\STATE // Track the weights
	\STATE Play arm $\cA_t = \argmin_{k\in [K]} \left(N_{t,k}-\sum_{\tau =0}^{t-1}\omega_{\tau,k}\right)$
    \STATE $N_{t+1,\cA_t} = N_{t,\cA_t}+1$ and $N_{t+1,k} = N_{t,k}$ for $k\neq \cA_t$
	\STATE $\bx_t=\argmax_{\bx\in \Sigma_d} \bx^\top \operatorname{LCB}(\blstar{}{}) \bomega_t$,
	\STATE where $\operatorname{LCB}(\blstar{}{})$ is calculated as in~\eqref{eq:lcb}
    \STATE // Feed optimistic loss
	\STATE Feed loss $\operatorname{LCB}(\blstar{}{})^\top \bx_t$ to $\cL$ and update $\operatorname{LCB}(\blstar{}{})$
\ENDFOR
\end{algorithmic}}
\end{algorithm}

\subsection{Analysis of \CG}


We show that \CG achieves a matching upper bound for the regret. We first show that the empirical estimation is valid with high probability at each round. Specifically, we have the following lemma.

\begin{lemma}\label{lemma:conc}
Define the following event:
\begin{align*}
    E_{1,t}\triangleq \left\{\forall k\in[K],i\in [d]: \abs{\hloss{i}{k}- \loss{i}{k}}\leq \sqrt{\dfrac{2\log(t)}{N_{k}}} \right\},
\end{align*}
where $N_{k}$ denotes the number of pulls of arm $k$. This event happens with probability at least $1- dK/t^2$:
\begin{align*}
    \PP {E_{1,t}} \geq 1- \frac{dK}{t^2}
\end{align*}
\end{lemma}

\begin{proof}
This is a direct application of the Hoeffding's Inequality with the union bound.
\end{proof}

With the Lemma above, we proceed to show that $\LCB(\blstar{}{})$ is a valid approximation for $\blstar{}{}$. Concretely, we have the following lemma.

\begin{lemma}
Assume that $E_{1,t}$ holds, we have,
\[
    \LCB_t(\blstar{}{})_{i,k} \le \lstar{i}{k}\,,
\]
\[
    \lstar{i}{k} \le \LCB_t(\blstar{}{})_{i,k} + 2\sqrt{\frac{2\log(t)}{N_{t,k}}} +2 \sqrt{\frac{2\log(t)}{N}}.
\]
\end{lemma}
\begin{proof}
This is a easy deduction of Lemma~\ref{lemma:conc}.
\end{proof}

\begin{restatable}{theorem}{restatecg}\label{thm:cg}
For $T\geq dK$, \CG achieves a $\tcO(T^{2/3})$ regret.
\end{restatable}

\begin{proof}
Recall the event $E_{1,t}$ defined in Lemma~\ref{lemma:conc}, we can decompose the regret as follows.
\begingroup
\allowdisplaybreaks
\begin{align*}
   &\EE{R_{\CG}(T)} = \EE{\max_{i\in[d]}\sum_{t=1}^T\lstar{i}{\cA_t}} - \wlstar{\bomega^\star}T \\
    \leq\ & KN+ \EE{\max_{i\in[d]}\sum_{t\geq KN}\lstar{i}{\cA_t}\II{\neg E_{1,t}}} \\\
    &\quad +\EE{\max_{i\in[d]}\sum_{t\geq KN}\lstar{i}{\cA_t}\II{E_{1,t}}} - \wlstar{\bomega^\star}T\:.
\end{align*}
\endgroup

For the second term, due to Lemma~\ref{lemma:conc}, we have
\begingroup
\allowdisplaybreaks
\begin{align*}
    &\EE{\max_{i\in[d]}\sum_{t\geq KN}\lstar{i}{\cA_t}\II{\neg E_{1,t}}} \\
    \leq\ & \EE{\sum_{t\geq KN}\II{\neg E_{1,t}}}\leq \sum_{t\geq KN}\frac{dK}{t^2} \leq \frac{\pi^2dK}{6}\,.
\end{align*}
\endgroup

For the third term, we first decompose the regret into a term related to strategies of both players and the tracking error. 
\begingroup
\allowdisplaybreaks
\begin{align}
    &\EE{\max_{i\in[d]}\sum_{t\geq KN}\lstar{i}{\cA_t}\II{E_{1,t}}} \nonumber\\
    \leq\ & \EE{\max_{i\in[d]} \sum_{k=1}^K \sum_{t= KN+1}^T \lstar{i}{k}\II{E_{1,t},\cA_t = k}}\nonumber\\
    \leq\ & \EE{\max_{i\in[d]} \sum_{k=1}^K\sum_{t=KN+1}^{T} \omega_{t,k}\lstar{i}{k}\II{E_{1,t}}}\nonumber\\
    & + \EE{\max_{i\in[d]}\sum_{k=1}^K \sum_{t=1+KN}^{T}\!\left(\II{\cA_t = k}\!-\! \omega_{t,k}\right)\II{E_{1,t}}\lstar{i}{k}}\nonumber\\
     \leq\ & \EE{\max_{i\in[d]} \sum_{k=1}^K\sum_{t=KN+1}^{T} \omega_{t,k}\lstar{i}{k}\II{E_{1,t}}} \nonumber\\
     & +\EE{\max_{i\in[d]}\sum_{k=1}^K \sum_{t=1+KN}^{T}\left(\II{\cA_t = k}- \omega_{t,k}\right)\lstar{i}{k}} \nonumber\\ 
     &+ \EE{\max_{i\in[d]}\sum_{k=1}^K \sum_{t=1+KN}^{T}\!\!\!\!\left( \omega_{t,k}\!-\!\II{\cA_t = k}\right)\lstar{i}{k}\II{\neg E_{1,t}}}\nonumber\\
    \leq\ &\EE{\sum_{t=KN+1}^{T} \max_{\bx\in \Sigma^d} \bx^\top \blstar{}{} \bomega_t\II{E_{1,t}}} + K + \frac{\pi^2dK}{6}\label{eq:tracking}\:,
\end{align}
where~\eqref{eq:tracking} adopts Lemma~15 by~\cite{garivier2016tracknstop}. The first term in~\eqref{eq:tracking} can be further estimated. 
\begin{align}
    &\EE{\sum_{t=KN+1}^{T} \max_{\bx\in \Sigma^d} \bx^\top \blstar{}{} \bomega_t\II{E_{1,t}}}\nonumber\\
    \leq\ & \EE{\sum_{t=KN+1}^{T}  \max_{\bx\in \Sigma^d} \bx^\top \LCB(\blstar{}{}) \bomega_t\II{E_{1,t}}}\nonumber \\
    & + \EE{2\sum_{t=KN+1}^{T}\sum_{k=1}^K \left(\sqrt{\frac{2\log(T)}{N_{t,k}}} \right)\omega_{t,k}}\nonumber \\
    &+ 2T\sqrt{\frac{2\log(T)}{N}} \nonumber\\
    \leq\ & \EE{\sum_{t=KN+1}^{T}  \max_{\bx\in \Sigma^d} \bx^\top \LCB(\blstar{}{}) \bomega_t\II{E_{1,t}}} \nonumber\\
    &+ 2(K^2+\sqrt{2KT})\sqrt{2\log(T)} + 2T\sqrt{\frac{2\log(T)}{N}} \label{eq:track2}\\
    \leq\ & \EE{\sum_{t=KN+1}^{T}  \max_{\bx\in \Sigma^d} \bx^\top \LCB(\blstar{}{}) \bomega_t} \nonumber\\
    &+ \!2(K^2\!+\!\sqrt{2KT})\sqrt{2\log(T)}\!+\! 2T\sqrt{\frac{2\log(T)}{N}}\!+\! \frac{\pi^2}{6}dK\nonumber\\
    \leq\ & \EE{\sum_{t=KN+1}^{T}  \bx_t^\top \LCB(\blstar{}{}) \bomega^\star} + \sqrt{T}\nonumber \\
    &+ \!2(K^2\!+\!\sqrt{2KT})\sqrt{2\log(T)}\!+\! 2T\sqrt{\frac{2\log(T)}{N}}\!+\! \frac{\pi^2}{6}dK\label{eq:orcreg}\:,
\end{align}  
where~\eqref{eq:track2} adopts Lemma~9 by \cite{degenne2019game} and~\eqref{eq:orcreg} uses the fact that $\cL$ has regret $\sqrt{T}$. Now we are only left to bound the first term in~\eqref{eq:orcreg}.
\begin{align}
    &\EE{\sum_{t=KN+1}^{T}  \bx_t^\top \LCB(\blstar{}{}) \bomega^\star}\nonumber\\
    \leq\ & \EE{\sum_{t=KN+1}^{T}  \bx_t^\top \blstar{}{} \bomega^\star} + \frac{\pi^2}{6}dK\nonumber\\
    \leq\ & (T-KN)  \max_{\bx\in \Sigma^d} \bx^\top \blstar{}{} \bomega^\star +  \frac{\pi^2}{6}dK\nonumber\\
    =\ & (T-KN)\wlstar{\bomega^\star} + \frac{\pi^2}{6}dK\nonumber\:,
\end{align}
\endgroup
Therefore, aggregating all the terms above we have the regret is upper bounded by $\cO(KN+T\sqrt{\log(T)/N})=\tcO(T^{2/3})$ by setting $N=(K^2T^2\log(T))^{1/3}$.
\end{proof}

\paragraph{Adaptive algorithm:}
The term $T^{2/3}$ comes from the trade-off between the exploration of $N$ rounds and the confidence bonus $\sqrt{2\log(T)/N}$. In fact, \CG does not need time horizon $T$ and forced exploration rounds $N$ as inputs. \CG can be easily refined by keeping each arm pulled for at least $t^{2/3}$ rounds at time $t$ and using a learner which is also adaptive, e.g., \AH~\citep{derooij2014hedge}.



\section{Discussion}\label{sec:discussion}

We studied a new setup of multi-armed bandit with vector losses. The main purpose of the paper was to investigate a framework for which we carefully constructed appropriate performance measures. We derived a problem-dependent lower bound of the sample complexity for best-arm identification and discussed how to design asymptotically optimal matching algorithms. We also derived a worst-case lower bound for regret minimization and designed a minimax game algorithm that achieves matching upper bound.


We are mainly interested in the maximum of different losses in this work. One possible future direction is to study how can we extend to a more general objective function instead of taking the maximum. Another interesting problem is to investigate whether we can derive a problem-dependent lower bound for regret minimization, for which the alternative bandit problem has a different optimal weight instead of a different single optimal arm.

\bibliographystyle{apalike}
\bibliography{Major}

\newpage
\onecolumn
\appendix

\section{More details on best-arm identification algorithm}\label{app:bai}
We provide an improved algorithm for best-arm identification. The idea is to view the problem again as a minimax game as for regret minimization, which is also a natural observation from the lower bound: given a bandit model $\bl$, at each time step a learner plays an arm, and a fictive opponent, tries to fool the learner by playing an alternative bandit model $\blambda$ with a different correct answer. Such framework allows to obtain algorithms that adapt to any structure with asymptotic optimality guarantees, and is extensively studied recently for best-arm identification~\citep{degenne2019game,degenne2019pure,menard2019lma,degenne2020game}.

In Algorithm~\ref{alg:bai}, we show one instance of such gamified sampling rule by~\cite{degenne2019pure}, adopted to our setting, along with the decision rule and the stopping rule we described in Section~\ref{sec:bai}. 

By applying the game scheme, we can actually approach the minimax saddle-point
\[
    \bomega^\star(\bl)=\argmax_{\bomega \in \Sigma_K}\inf_{\bm\lambda \in\Alt(\bl)} \left( \sum_{k\in [K]}\sum_{i\in [d]} \omega_k d(\loss{i}{k},\Alti{i}{k})\right)
\]
step by step by leveraging an iterative algorithm for both the real learner and the fictive opponent. 

We first need to implement a learner $\cL_\bomega^i$ for each possible answer $i\in\cI$, for which we can apply choose to use \AH again (as for regret minimization). Not to enter into details, \AH is a regret minimizing algorithm of the exponential weights family, that achieves a $\cO(\sqrt{T})$ regret for bounded losses~\citep{derooij2014hedge}. At each time step, we get a weight vector $\bomega_t$ from $\cL_\bomega^{i_t}$ where $i_t$ is the empirical best answer. However, a bandit algorithm cannot play a fraction vector. We can incorporate a tracking procedure (see line 20 in Algorithm~\ref{alg:bai}) to circumvent this difficulty. For the opponent learner, we choose to use the best response, that is the most confusing model as formalized in Line 14 of the pseudo-code.

The present procedure involves also an optimization problem, but simpler than that of \DT, and is computationally more feasible in practice (note that other combination of sub-algorithms for both opponents are possible, see e.g.~\citealt{degenne2019game,menard2019lma}).

\begin{algorithm}[t]
   \caption{A gamified algorithm for best-arm identification}
   \label{alg:bai}
\begin{algorithmic}[1]
   \STATE {\bfseries Input:} Learners for each possible answer $(\cL^i_\bomega)_{i\in\cI}$, threshold function $\beta(\cdot,\delta)$, exploration bonus $f(\cdot)$, number of forced exploration rounds $N$
   \STATE pull each arm $N$ rounds
   \FOR{$t > KN$}
    \STATE \textit{// Stopping rule}
    \IF{at time $t$, we have 
    \begin{align*}
        \max_{k\in[K]} \inf_{\blambda \in\Alt(\hat{\bl}_{t-1})} \left( \sum_{k\in [K]}\sum_{j\in [d]} N_{t-1,k} d(\hloss{j}{t-1,k},\Alti{j}{k})\right) > \beta(t-1,\delta)
    \end{align*}}
      \STATE {\bfseries stop} and {\bfseries return} 
      \STATE $\hi = \istar(\tilde{\bl}_{t-1})$  
    \ENDIF
    \STATE \textit{// Best answer}
    \STATE $i_t = \istar(\tilde{\bl}_{t-1})$
    \STATE \textit{// The learner plays}
    \STATE Get $\bomega_{t}$ from $\cL^{i_t}_\bomega$ and update $\bW_{t}=\bW_{t-1}+\bomega_{t}$
    \STATE \textit{// Best response from the nature}
    \STATE \[\blambda_{t} \in \argmin_{\blambda \in\Alt(\hat{\bl}_{t-1})} \left( \sum_{k\in [K]}\sum_{j\in [d]} \omega_{t,k} d(\hloss{j}{t-1,k},\Alti{j}{k})\right)\]
    \STATE \textit{// Feed optimistic losses}
    \STATE For $k\in[K]$, let
    \STATE \[U_{t,k}=\max\left(\frac{f(t-1)}{N_{t-1,k}},\max_{{\xi}\in\{c_{t,k},d_{t,k}\}}\sum_{j\in[d]}d(\xi,\Alti{j}{t,k})\right)\]
    \STATE Feed $-\sum_{k\in[K]}\omega_k U_{t,k}$ to learner $\cL_\bomega^{i_{t}}$
    \STATE \textit{// Track the weights}
    \STATE Pull $\cA_{t}\in \argmin_{k\in [K] } N_{t-1,k} - W_{t,k}$
   \ENDFOR
\end{algorithmic}
\end{algorithm}
\paragraph{Sample complexity.} 

In this paper, since we assume that our model $\bl \in [0,1]^{d\times K}$, there exists constants $\gamma_{\bl}$ and $D_{\bl}$, that only depend on the model $\bl$, such that for all $y\in[0,1]$, the function $x\mapsto d(x,y)$ is $\gamma_{\bl}$-Lipschitz on $[0,1]$ and $d(x,y)\leq D_{\bl}$ (see Appendix F of~\citealt{degenne2019game} for detailed discussions).

According to Theorem 2 of~\cite{degenne2019game}, the sample complexity of Algorithm~\ref{alg:bai} at the stopping time $\tau_\delta$ is bounded by $T_\delta$ as defined in Theorem~\ref{thm:bai}, which is a non-asymptotic bound that depends on regrets incurred by both the \AH learner and the best response learner. For \AH, the regret incurred is $R^{k}(t)=\sqrt{t\log(dK)}\log(t)$, and the best-response learner has zero-regret: $R^{\blambda}(t)\leq 0$. And with $\beta(t,\delta)\approx\log(1/\delta)+o(t)$, the asymptotic optimality of Algorithm~\ref{alg:bai} is also retained. We do not intend to reproduce the proof here since it can be (almost) adopted directly from the proof of Theorem 2 by~\cite{degenne2019game} (up to a factor of $d$).

\begin{theorem}\label{thm:bai}
The sample complexity of Algorithm~\ref{alg:bai} on model $\bl$ is
\[
    \E_\bl[\tau_{\delta}] \leq T_{\delta} + \texttt{CST}\,.
\]
The quantity $T_{\delta}$ is defined as
\begin{gather*}
    T_{\delta} \eqdef \max\left\{t\in\NN: t\leq \beta(t,\delta)/D_{\bl} +C_{\bl}(R^{\blambda}(t)+R^k(t)+\cO(t\log(t)))\right\}
\end{gather*}
where $C_{\bl}$\footnote{See Appendix D of~\cite{degenne2019game} for an exact definition.} depends on the model $\bl$.
\end{theorem}
\section{A simple algorithm for regret minimization and its analysis}\label{app:cp}

\subsection{A combinatorial lemma}

We now describe in detail an algorithm based on the proof idea of Theorem~\ref{thm:lb}. As we stressed, it is more viable to consider the proportion (weight) of arm pulls, in particular the optimal weight of arm pulls for regret minimization. To simplify the problem, we first show that the optimal distribution is a linear combination over $d$ arms. 



\begin{lemma}\label{lemma:comb}
In the case of $d$ losses, there exists a $\bomega^\star$ such that it has at most $d$ non-zero elements.
\end{lemma}
\begin{proof}
We first define the quadrant $H^+ \eqdef \left\{\bx|x_i\geq 0\right\}$ and we define an addition operation of two sets $A$ and $B$ as $A + B \eqdef \{a+b|a\in A, b\in B\}$. For any compact set $X$ we note that 
\[
    \inf\limits_{\bx\in X}\max(x_1,...,x_d) =  \inf\limits_{\by\in (X+H^+)\bigcap \operatorname{Diag}} y_1\,,
\]
where $\operatorname{Diag} \eqdef \left\{\textbf{x}|x_1 = x_2 = ... = x_d\right\}$. Let $A \eqdef \operatorname{Conv}\left(\left\{\blstar{}{k} \right\}_{k\in[K]}\right)$ be the convex hull over the relative losses of all $K$ arms. Then we have 
\[
    \wlstar{\bomega^\star} = \inf\limits_{\by\in (A+H^{+})\bigcap \operatorname{Diag}} y_1 \,.
\]
Therefore, there exists $\bomega\in A$ and $h\in H^+$ such that $\wlstar{\bomega^\star} = \bomega^\top \blstar{(i)}{}+h_i$ for all $i\in[d]$. Moreover, it is obvious that there is at least one $i\in [d]$ such that $h_i = 0$. Thus, here $\bomega$ is an optimal weight, i.e., $\wlstar{\bomega^\star}=\wlstar{\bomega}$. Furthermore, the vector $\bomega^\top \blstar{(i)}{}$ is not an interior point of $A$, since this would enable a $(-1,...,-1)$ direction translation, thus $\bomega^\top \blstar{(i)}{}$ is on a surface of $A$, a convex hull of finite points in a $d-1$ dimension space, we conclude that there exists such an $\bomega$ with at most $d$ non-zero elements.

\end{proof}

\subsection{A straightforward algorithm}


We assume that $K\geq d$. According to Lemma~\ref{lemma:comb}, we only need to consider linear combinations of $d$ arms. We define a combinatorial arm $\bc\in C$ as a set of $d$ arms where $C=\{\{c_1,...,c_d\}\subseteq [K]|c_1<c_2<\ldots<c_d\}$.  For all $\bc\in C$, we are interested in the quantity $\lstarC{(i)}{\bc} \triangleq \min_{\balpha\in \Sigma_d} \balpha^\top \blstar{(i)}{\bc}$, where $\blstar{(i)}{\bc} = (\lstar{i}{c_1}, \dots , \lstar{i}{c_d} )$ is the vector of the $i$-th relative losses of arm set $\bc$. We further denote by $\lstarC{}{\bc} = \max_{i\in [d]} \lstarC{(i)}{\bc}$ the $\ell^\infty$-norm of the relative loss and $\loss{i}{\bc}=\lstarC{(i)}{\bc} + \loss{i}{\star_i}$ the absolute loss of the combinatorial arm.

A straightforward idea of algorithm for regret minimization is to track the values of $\lstarC{}{\bc}$ for every $\bc\in C$. We thus need to have a good estimate of the relative loss of all $d$ combinations of arms. We propose \underline{C}ombinatorial \underline{P}lay (\CP) as shown in Algorithm~\ref{alg:rm}. The empirical relative loss of arm $\bc\in C$ w.r.t. $\hat{\star}_i$ is defined as
\begin{align}
    \hhwlstar{\bc} = \max_{i\in[d]}\min_{\balpha\in\Sigma_d}  \balpha^\top \hblhstar{(i)}{\bc} \,.\label{eq:eprl}
\end{align}
Let $\hat{\balpha}_{\bc}$ denote the value of $\balpha$.

Our algorithm thus chooses among $\bc \in C $ and calculates the empirical optimal allocation $\hat{\balpha}_\bc \in \Sigma_d$ among $\bc$. Then we use the tracking procedure from the literature (see e.g.~\citealt{garivier2016tracknstop}) to decide which real arm to pull.

\begin{algorithm}[ht] \caption{The algorithm of \CP}\label{alg:rm}
{\begin{algorithmic}[1]
\STATE {\bfseries Input:} time horizon $T$ and number of forced exploration rounds $N$
\STATE pull each arm $N$ rounds
\STATE $\hat{\star}_i = \argmin_{k\in[K]} \hloss{i}{k}$ for $i\in [d]$
\STATE for all $\bc\in C$, we calculate its estimate $\hhwlstar{\bc}$ and its optimal allocation $\hat{\balpha}_{\bc}$ based on Eq.~\eqref{eq:eprl}
\STATE $\hat{\bc} \in \argmin_{\bc\in C}\hhwlstar{\bc}$ and the corresponding optimal allocation $\hat{\balpha}_{\hat{\bc}}$.
\FOR{$t=KN+1,\cdots,T$}
	\STATE Pull arm $\cA_t$ according to probability distribution $\hat{\balpha}_{\hat{\bc}}$ over ${\hat{\bc}}$.
\ENDFOR
\end{algorithmic}}
\end{algorithm}

\subsection{Analysis of \CP}\label{sec:lb.cp}
We analyze \CP in this section. Our main result is stated below.
\begin{restatable}{theorem}{restatecp}\label{thm:cp}
Assume that $\hloss{i}{k}\in[0,1]$ for all $k\in[K]$ and $i\in [d]$, \CP achieves a $\tcO(T^{2/3})$ regret.
\end{restatable}
\begin{proof}
First, according to Lemma~\ref{lemma:conc}, we have $\PP{E_{1,t}}\geq 1-dK/t^2$. When $E_{1,t}$ holds, we have for all $\bc\in C$
\begin{align*}
\hhwlstar{\bc} = \max_{i\in[d]}\min_{\balpha\in\Sigma_d}  \balpha^\top  \hblhstar{(i)}{\bc}   \geq \max_{i\in[d]}\min_{\alpha\in\Sigma_d}    \left(\balpha^\top  \blstar{(i)}{\bc}   - 2\sqrt{\frac{2\log(t)}{N}}\right) =   \lstarC{}{\bc} - 2\sqrt{\frac{2\log(t)}{N}}\,.
\end{align*}
And similarly, we have
\begin{align*}
\hhwlstar{\bc} \leq  \lstarC{}{\bc} + 2\sqrt{\frac{2\log(t)}{N}}\,.
\end{align*}

Then the regret is
\begingroup
\allowdisplaybreaks
\begin{align*}
    \EE{R_{\CP}(T)} &\leq KN +\sum\limits_{t=KN+1}^{T} \EE{ \II{\neg E_{1,t}}} +  (\wlstar{\bc}-\wlstar{\bomega^\star})(T-KN)\EE{\II{\forall t, E_{1,t}}} \\
    &\leq KN + \sum\limits_{t=KN+1}^{T} \frac{dK}{t^2} + (\hhwlstar{\bc}- \min_{\bc} \lstarC{}{\bc})(T-KN)\EE{\II{\forall t, E_{1,t}}}   \\
    &+ 2\sqrt{\frac{2\log(t)}{N}} (T-KN)\EE{\II{\forall t, E_{1,t}}} \\
    &\leq KN + \frac{d}{N} +  \left(\hhwlstar{\bc}- \min_{\bc} \left(\hhwlstar{\bc} - 2\sqrt{\frac{2\log(t)}{N}} \right)  \right)(T-KN)\EE{\II{\forall t, E_{1,t}}} \\
    &+  2\sqrt{\frac{2\log(t)}{N}} (T-KN)\EE{\II{\forall t, E_{1,t}}} \\
    &\leq KN + \frac{d}{N} + 4\sqrt{\frac{2\log(T)}{N}} (T-KN)\:,
\end{align*}
\endgroup
which completes the proof by setting $N= ({32T^2\log(T)}/{K^2})^{\frac{1}{3}}$.
\end{proof}

\end{document}